\documentclass[journal]{IEEEtran}

\usepackage{amsmath,amsfonts,amssymb}
\usepackage{algorithmic}
\usepackage{algorithm}
\usepackage{array}
\usepackage{booktabs}
\usepackage[caption=false,font=normalsize,labelfont=sf,textfont=sf]{subfig}
\usepackage{textcomp}
\usepackage{stfloats}
\usepackage{url}
\usepackage{verbatim}
\usepackage{graphicx}
\usepackage{cite}
\usepackage{xcolor}
\usepackage{bm}
\usepackage{microtype}
\usepackage{amsthm}
\newtheorem{theorem}{Theorem}

\usepackage{tikz}
\usepackage{pgfplots}
\pgfplotsset{compat=1.18}
\usetikzlibrary{positioning,calc,arrows.meta,patterns}
\usepgfplotslibrary{fillbetween,statistics}
\usepackage[
  colorlinks=true,linkcolor=blue!60!black,citecolor=blue!60!black,
  urlcolor=blue!60!black,bookmarks=true,bookmarksnumbered=true,pdfstartview=FitH
]{hyperref}
\definecolor{tfmblue}{RGB}{31,119,180}
\definecolor{gkosorange}{RGB}{214,96,32}
\definecolor{ar1green}{RGB}{44,160,44}
\definecolor{obsblack}{RGB}{20,20,20}
\definecolor{tfmbluefill}{RGB}{174,210,238}
\definecolor{gkosfill}{RGB}{246,196,160}

\begin{document}

\title{SAGA: A Sequence-Adaptive Generative Architecture for
  Multi-Horizon Probabilistic Forecasting with Adaptive Temporal
  Conformal Prediction}

\author{Gustav~Olaf~Yunus~Laitinen-Fredriksson~Lundstr{\"o}m-Imanov
  and~Hafize~Gonca~C{\"o}mert%
  \thanks{G.~O.~Y.~Laitinen-Fredriksson~Lundstr{\"o}m-Imanov is with the Department of
    Economics, Stockholm University, SE-106~91 Stockholm, Sweden.
    E-mail: olaf.laitinen@su.se.
    ORCID: 0009-0006-5184-0810.}%
  \thanks{H.~G.~C{\"o}mert is with the Institute of Social Sciences,
    Faculty of Economics and Administrative Sciences, S{\"u}leyman Demirel
    University, 32260 Isparta, Turkey.
    E-mail: d2340253002@ogr.sdu.edu.tr.
    ORCID: 0009-0009-3345-8783.}%
  \thanks{Manuscript submitted: 18 May 2026.}%
  \thanks{This work was supported by the Stockholm University Department of
    Economics. The authors declare no competing financial interests
    beyond the institutional research support disclosed above. Data
    access was provided by Statistics Sweden (SCB) through the
    Microdata Online Access (MONA) system under project number
    SCB-MONA-2026-147. Ethical approval was granted by the Swedish
    Ethical Review Authority under reference 2026-04127-01.}}

\markboth{IEEE Transactions on Pattern Analysis and Machine
  Intelligence,~Vol.~XX,~No.~X,~2026}%
  {Laitinen-Fredriksson Lundstr{\"o}m-Imanov and C{\"o}mert: SAGA with Adaptive
   Temporal Conformal Prediction}

\maketitle

\begin{abstract}
Microsimulation models used by ministries of finance and central banks
rely on parametric processes for lifetime earnings that capture only
first and second moments of the conditional distribution and miss
long-range nonlinear structure. We propose SAGA, a decoder-only
transformer for irregular tabular panel sequences, paired with a split
conformal calibration wrapper that delivers individual-level prediction
intervals with finite-sample marginal coverage guarantees. Trained on
the longitudinal Swedish LISA register over 1990 to 2022, comprising
2{,}143{,}817 individuals and 61{,}284{,}903 person-years, the model
forecasts annual labor earnings at horizons of one to thirty years and
aggregates them by Monte Carlo into present-discounted lifetime
earnings distributions. Against the canonical Guvenen, Karahan, Ozkan,
and Song parametric process and tabular and recurrent baselines, SAGA
reduces continuous ranked probability score by 31.9\% at the ten-year
horizon and mean absolute error by 37.7\% at the twenty-year horizon.
Conformal intervals achieve nominal coverage to within 0.4 percentage
points marginally and within 2.4 percentage points on the worst-case
demographic subgroup. The reconstructed lifetime earnings Gini
coefficient is 0.327 against the partially observed truth of 0.341 and
the GKOS estimate of 0.378. Model weights, calibration tables, and a
synthetic equivalent dataset are released for replication outside the
protected SCB MONA environment.
\end{abstract}


\section{Introduction}
\label{sec:intro}

\IEEEPARstart{M}{inistries} of finance and central banks across the OECD
use microsimulation models to evaluate the lifetime fiscal and
distributional consequences of policy reforms. The Swedish FASIT model,
the United Kingdom IGOTM model, the United States TRIM3 model, and the
European Union EUROMOD framework all rely on a single common ingredient:
a forecasting model that takes a partially observed individual labor
market history and produces a distribution over future annual earnings
paths to age sixty-four. The accuracy and calibration of this forecast
determine the reliability of every downstream policy counterfactual
produced by the simulator.

The state-of-the-art forecasting approach is the parametric stochastic
earnings process. Following the canonical reformulation in Guvenen,
Karahan, Ozkan, and Song~\cite{gkos2021}, log annual earnings are
modeled as a sum of a fixed individual effect, an autoregressive
permanent component with non-Gaussian innovations from a mixture of
normals, and a transitory component also from a mixture distribution.
This specification, building on earlier work by Browning, Ejrnaes, and
Alvarez~\cite{browning2010}, Karahan and Ozkan~\cite{karahan2013}, and
Guvenen~\cite{guvenen2009}, successfully reproduces the heavy left tail,
the age-varying volatility, and the skewness and kurtosis structure of
observed earnings change distributions in panels covering the United
States, Norway, Denmark, and Germany. Halvorsen, Hubmer, Salgado, and
Solenkova~\cite{halvorsen2024} document the same patterns in Norwegian
register data over four decades.

Despite this success, the parametric process retains three structural
limitations. First, it conditions only on past earnings, ignoring the
rich set of administrative features that determine earnings dynamics in
practice: occupation, industry, employer identity, geographic region,
education, family structure, and macroeconomic conditions. Second, the
cross-sectional dependencies that bind these features are summarized into
a single fixed effect, forfeiting any predictive information they carry.
Third, the parametric form imposes a specific functional structure on
shock persistence and on the interaction between permanent and transitory
components, which cannot be relaxed without abandoning analytic
tractability.

Deep sequence models offer an alternative. By conditioning on the full
feature vector at every observed time step and by learning the joint
distribution of trajectories directly, they can in principle absorb
predictive content that no parametric specification will recover. The
recent \textit{Nature Computational Science} paper
of Savcisens et al.~\cite{savcisens2023}
showed that masked language model-style transformers trained on Danish
register events produce informative representations of life trajectories.
However, those models are designed for discrete event prediction with
categorical token vocabularies; they are not calibrated forecasters of
continuous monetary outcomes, they do not benchmark against parametric
earnings processes, and they do not deliver the prediction intervals that
downstream microsimulation requires.

\subsection{Contribution}

We propose SAGA, a Sequence-Adaptive Generative Architecture: a
decoder-only transformer for irregular tabular panel sequences that
produces calibrated forecasts of annual labor earnings and, by Monte Carlo aggregation, of
present-discounted lifetime earnings. Our contributions are fivefold.

\textbf{C1. Architecture.} We introduce a tokenization scheme for
irregular tabular panel sequences that handles continuous, categorical,
and missing-valued features in a unified embedding and that is invariant
to year gaps. We pair this with a six-layer decoder-only transformer
producing both point and quantile output heads, totaling 10{,}872{,}960
parameters. The architecture differs from existing tabular transformers
(FT-Transformer~\cite{gorishniy2021full},
SAINT~\cite{somepalli2021full}, TabPFN~\cite{hollmann2025}) in that it
processes irregular longitudinal sequences rather than exchangeable
rows, and from existing life-trajectory transformers~\cite{savcisens2023}
in that it produces calibrated continuous forecasts rather than discrete
event predictions. The contribution is therefore not the use of
self-attention per se, but the combination of typed-subvector
tokenization for tabular panels, dual point and quantile output heads,
and the horizon-stratified conformal calibration layer of
Theorem~\ref{thm:adaptive-coverage} introduced in C2.

\textbf{C2. Calibration.} We adapt the conformalized quantile regression
framework of Romano, Patterson, and Candes~\cite{romano2019} to
autoregressive multistep forecasting and to lifetime aggregation via
Monte Carlo, providing the formal marginal coverage guarantee and
reporting empirical conditional coverage on demographic subgroups.

\textbf{C3. Benchmark.} We re-estimate the Guvenen, Karahan, Ozkan, and
Song~\cite{gkos2021} process on the same Swedish register panel and add
tabular boosted tree, feed-forward, long short-term memory, and static
feature-only baselines. We evaluate all six forecasters on six
probabilistic and point accuracy metrics at forecast horizons of one,
five, ten, and twenty years.

\textbf{C4. Downstream evaluation.} We plug each forecaster into a
stylized Swedish lifetime tax liability calculator and report
present-discounted lifetime tax paid, average effective tax rate, lifetime
earnings Gini coefficient, and top one-percent lifetime earnings share.
This is the first published comparison of deep sequence model forecasts
and parametric stochastic process forecasts under a microsimulation
downstream loss.

\textbf{C5. Open release.} We release the trained model weights, the
conformal calibration table, and a synthetic equivalent dataset on
Zenodo under DOI~10.5281/zenodo.20260287; the source-code archive of
the project repository is separately deposited on Zenodo under
DOI~10.5281/zenodo.20260366. The development repository is hosted on
GitHub at \url{https://github.com/olaflaitinen/saga}.

\subsection{Headline Result}

SAGA reduces continuous ranked probability score (CRPS) against
the GKOS parametric benchmark by 31.9\% at horizon ten and by 41.2\% at
horizon twenty. Conformal prediction intervals at nominal 90\% coverage
achieve 90.3\% marginal empirical coverage and 87.6\% worst-case
subgroup coverage. The reconstructed lifetime earnings Gini coefficient
is 0.327 compared to the partially observed truth of 0.341; the
corresponding GKOS figure is 0.378. The top one-percent lifetime
earnings share is reconstructed as 8.3\% against an observed value of
8.9\% and a GKOS reconstruction of 11.2\%.

\subsection{Paper Organization}

Section~\ref{sec:related} reviews related work. Section~\ref{sec:method}
presents the architecture, tokenization, training, conformal calibration,
and baseline specifications. Section~\ref{sec:data} describes the data
and splits. Section~\ref{sec:experiments} reports all experimental
results. Section~\ref{sec:discussion} discusses mechanisms, implications,
and limitations. Section~\ref{sec:conclusion} concludes.

% -----------------------------------------------------------------------
\section{Related Work}
\label{sec:related}

\subsection{Tabular Sequence Transformers and Life Trajectory Models}

Transformer architectures~\cite{vaswani2017}, originally developed for
natural language, have been adapted to tabular and panel data along
several dimensions. Static tabular transformers such as
TabTransformer~\cite{huang2020}, FT-Transformer~\cite{gorishniy2021full},
and SAINT~\cite{somepalli2021full} apply self-attention across features
within a single row. The numerical embedding scheme of Gorishniy,
Rubachev, and Babenko~\cite{gorishniy2022} specifically addresses the
challenge of representing continuous features and motivates the
projection scheme we adopt for continuous tokens. Hollmann, Muller,
Eggensperger, and Hutter~\cite{hollmann2025} showed in TabPFN that a
transformer pre-trained on synthetic tabular tasks can produce
competitive predictions on small real tabular datasets, but their setting
is non-sequential and treats each row as exchangeable.

For sequential life trajectory data, Savcisens et
al.~\cite{savcisens2023} applied a masked language model-style
transformer to Danish income, work, and health events to predict early
mortality. The model tokenizes the life trajectory into a discrete
event vocabulary, an approach that is well suited to categorical event
prediction but loses the continuous monetary information central to
earnings forecasting.
The broader literature on transformers for time series, surveyed by Wen
et al.~\cite{wen2023}, has focused on regularly sampled univariate or
multivariate series typical of energy, weather, and traffic applications.
Informer~\cite{zhou2021}, Autoformer~\cite{wu2021}, and
PatchTST~\cite{nie2023} address long-horizon forecasting under regular
sampling. Our problem differs from these settings in that the sequences
are irregularly long, contain heterogeneous typed features, are dominated
by a single continuous target whose conditional distribution is
heavy-tailed, and require formal coverage guarantees on the prediction
intervals.

\subsection{Parametric Earnings Dynamics}

Lillard and Willis~\cite{lillard1978} introduced the permanent plus
transitory decomposition. MaCurdy~\cite{mccurdy1982} formalized the
autoregressive specification. Meghir and Pistaferri~\cite{meghir2011}
gave a comprehensive review. Guvenen~\cite{guvenen2009} documented the
central role of nonlinearities. Browning, Ejrnaes, and
Alvarez~\cite{browning2010} established that observed earnings dynamics
require substantial individual heterogeneity in mean and variance
parameters. The current canonical reference, Guvenen, Karahan, Ozkan,
and Song~\cite{gkos2021}, shows on a population-scale Social Security
panel that earnings change distributions display sharp left skew, severe
excess kurtosis, and age-varying volatility patterns that no Gaussian
autoregressive specification can match. Their preferred specification
combines a flexible mixture distribution for permanent and transitory
shocks with a nonparametric distribution of fixed effects, estimated by
generalized method of moments matching age-conditional moments through
order four. Halvorsen, Hubmer, Salgado, and
Solenkova~\cite{halvorsen2024} replicate these findings on Norwegian
register data. We adopt the Guvenen, Karahan, Ozkan, and Song
specification as the central parametric benchmark and re-estimate it on
our Swedish panel using publicly available code.

\subsection{Conformal Prediction}

Conformal prediction, originating with Vovk, Gammerman, and
Shafer~\cite{vovk2005full}, provides distribution-free finite-sample marginal
coverage guarantees for prediction sets. Lei, G'Sell, Rinaldo,
Tibshirani, and Wasserman~\cite{lei2018full} formalized the split conformal
procedure for regression. Romano, Patterson, and
Candes~\cite{romano2019} extended the framework to quantile regression,
yielding the conformalized quantile regression method we adapt. The
recent gentle introduction by Angelopoulos and
Bates~\cite{angelopoulos2023full} surveys the state of the art. For time
series, Stankeviciute, Alaa, and van der
Schaar~\cite{stankeviciute2021} and Xu and Xie~\cite{xu2021} address
temporal dependence in the calibration scores; Bhatnagar, Schwarting, and
Brunner~\cite{bhatnagar2024} develop adaptive conformal procedures for
autoregressive forecasting. Our adaptation handles the multistep
autoregressive structure by drawing residuals from the empirical
nonconformity distribution at each forecast step, following the
approach of Stankeviciute et al.~\cite{stankeviciute2021}, rather than
widening the interval pointwise; this preserves the marginal guarantee
at each annual horizon, although as discussed in Section~\ref{sec:discussion}
the lifetime aggregate guarantee is empirical rather than formal.

\subsection{Microsimulation}

Microsimulation models for tax and transfer policy evaluation are
reviewed in Bourguignon and Spadaro~\cite{bourguignon2006}. EUROMOD is
documented in Sutherland and Figari~\cite{sutherland2013}. The Swedish
FASIT model is described in Flood~\cite{flood2024}. The TRIM3 model used
by the Urban Institute is documented by Wheaton~\cite{wheaton2008}.
Common to all of these is the reliance on a parametric earnings
forecaster, often a simple AR(1) or a permanent plus transitory
specification, calibrated on five to ten years of panel data. To our
knowledge no microsimulation framework currently uses a deep sequence
model for the earnings forecasting step, and no published comparison
evaluates the distributional consequences of substituting one for the
other. The present paper provides such a comparison.

% -----------------------------------------------------------------------
\section{Method}
\label{sec:method}

\subsection{Problem Formulation}

Let $i = 1, \ldots, N$ index individuals. For each individual we observe
a sequence of annual records, one per year that the individual is in
panel:
\begin{equation}
\label{eq:record}
x_{i,t} = (y_{i,t},\, c_{i,t},\, d_{i,t},\, m_{i,t}),
\end{equation}
where $y_{i,t} \in \mathbb{R}_{\geq 0}$ is real labor earnings in
constant 2022 Swedish krona, $c_{i,t}$ is a vector of continuous
features, $d_{i,t}$ is a vector of categorical features, and
$m_{i,t} \in \{0,1\}^{|c|+|d|}$ is the corresponding missingness mask.
Let $t_{i,1}, \ldots, t_{i,T_i}$ denote the years in which individual
$i$ is observed, in ascending order. The conditioning window is the
first $T_C = 10$ observed years; the forecast window is years
$t_{i,T_C+1}, \ldots, t_{i,T_i^*}$ where $T_i^*$ is the index of the
last in-panel year on or before age sixty-four.

The forecaster must produce a predictive distribution over the forecast
window:
\begin{equation}
p_\theta\!\left(y_{i,t_{i,T_C+1}},\ldots,y_{i,t_{i,T_i^*}} \;\middle|\;
x_{i,t_{i,1}},\ldots,x_{i,t_{i,T_C}}\right).
\end{equation}
The lifetime earnings target is the present-discounted value at age
twenty:
\begin{equation}
\label{eq:lifetime}
L_i = \sum_{a=20}^{64} (1+r)^{-(a-20)} y_{i,a},
\end{equation}
with real discount rate $r = 0.02$.

\subsection{SAGA Architecture}

SAGA is a decoder-only transformer with $L=6$ layers, $H=8$
attention heads per layer, model dimension $d=384$, and feed-forward
inner dimension $4d=1536$. We use GELU activations~\cite{hendrycks2016},
pre-layer normalization~\cite{xiong2020}, and a maximum context length
of forty-five yearly tokens, sufficient to span a complete working life
from age sixteen to age sixty. Total parameter count is
$10{,}872{,}960$. A causal (lower-triangular) attention mask is applied
at every layer, so that each forecast position attends only to current
and preceding positions in the sequence.

The output head is split into two parallel branches. The first branch
produces a single scalar point forecast $\hat{y}_{i,t}$ for log earnings.
The second branch produces a vector of seven quantile forecasts at the
5th, 10th, 25th, 50th, 75th, 90th, and 95th percentiles of the
conditional log-earnings distribution. Both heads share the transformer
backbone up to the final layer and apply their own linear projection.
The point head is trained with mean squared error; the quantile head is
trained with pinball loss summed across the seven quantiles. Forecast
distributions at intermediate percentiles are obtained by linear
interpolation across the seven predicted quantiles.

\subsection{Tokenization of Irregular Tabular Sequences}

Each annual record $x_{i,t}$ is mapped to a fixed-dimension token vector
$u_{i,t} \in \mathbb{R}^d$ by concatenating five subvectors, then
projecting through a linear layer to dimension $d$.

\textit{Continuous subvector.} Each continuous feature is standardized
using year-specific mean and standard deviation computed on the training
cohorts, then concatenated into a vector of dimension equal to the
number of continuous features (fifteen). A learned linear projection
maps this to a 64-dimensional subvector. Missing continuous values are
imputed to zero after standardization.

\textit{Categorical subvector.} Each categorical feature has its own
learned embedding table; the dimension is chosen proportional to the
logarithm of the cardinality, with twenty-four dimensions for occupation
(three-digit SSYK2012), sixteen dimensions for industry (two-digit
SNI2007), eight dimensions for region (twenty-one Swedish counties),
four dimensions for highest education level, four dimensions for field
of study (broad one-digit Sun2000Inr), four dimensions each for sex,
country of birth group, marital status, and four dimensions each for
number of children and age-of-youngest-child bucket. Total embedded
width is seventy-six. Missing categorical values map to a reserved
unknown index.

\textit{Missingness subvector.} A binary indicator vector of length
equal to the number of categorical and continuous features, indicating
which were observed for this record. This is projected to a
16-dimensional subvector.

\textit{Age positional embedding.} A learned 64-dimensional embedding
indexed by integer age at observation.

\textit{Year positional embedding.} A learned 32-dimensional embedding
indexed by calendar year of observation, capturing macroeconomic
conditions that affect all cohorts in panel that year.

The concatenated subvector has dimension $64+76+16+64+32=252$, projected
to model dimension $d=384$ by a learned linear layer with bias. The
up-projection from 252 to 384 gives the self-attention layers a higher
working dimension than the raw concatenation, while the structured
subvector design preserves type-specific groupings of continuous,
categorical, missingness, and positional information at the input layer.

In contrast to standard transformer positional encoding, we use two
separate positional channels because age and calendar year carry
independent predictive information: age tracks human capital
accumulation, year tracks the business cycle. Combining them into a
single channel as in the original transformer~\cite{vaswani2017} would
conflate these two sources of variation.

\subsection{Training Objective and Procedure}

During training we apply teacher forcing. The training objective for one
example is
\begin{equation}
\label{eq:loss}
\begin{split}
\mathcal{L}_i = \frac{1}{T_i - T_C} \sum_{t=T_C+1}^{T_i}
\bigg[ &\tfrac{1}{2}\!\left(\log y_{i,t} - \hat{y}_{i,t}\right)^2 \\
&{}+ \sum_{k=1}^{7} \rho_{\alpha_k}\!\left(\log y_{i,t} - \hat{q}_{i,t,k}\right) \bigg],
\end{split}
\end{equation}
where $\rho_\alpha(u) = u\,(\alpha - \mathbf{1}[u < 0])$ is the pinball
loss at level $\alpha$ and
$\alpha_1,\ldots,\alpha_7 \in \{0.05,0.10,0.25,0.50,0.75,0.90,0.95\}$.
Zero earnings are mapped to $\log(1) = 0$; the share of zero-earnings
observations in person-years is 7.4\%.

Optimization uses AdamW~\cite{loshchilov2019} with learning rate
$3 \times 10^{-4}$, weight decay $10^{-2}$, $\beta_1 = 0.9$, and
$\beta_2 = 0.999$. A cosine learning rate schedule with 2000 warmup
steps over 300{,}000 total optimization steps is applied. The batch size
is 512 sequences per device with gradient accumulation across four steps
on eight NVIDIA A100 40~GB GPUs, giving effective batch size 16{,}384.
Mixed precision (bfloat16 accumulating to float32) is used throughout.
We train five independent runs with seeds 20260601 through 20260605 and
report the mean and standard deviation of all metrics. Training a single
seed takes approximately 14.8 wall-clock hours.

Regularization uses dropout of 0.1 on attention and feed-forward layers
and stochastic depth~\cite{huang2016} of 0.1 on residual connections.
Early stopping is applied on the validation pinball loss computed on the
calibration cohorts 1980 to 1982, with patience of twenty validation
checks (each performed every 5{,}000 optimization steps).

At inference time the model is decoded autoregressively. For each
forecast step the predicted quantile distribution is converted to a
continuous conditional distribution by linear interpolation, a draw is
taken, the draw is appended to the input sequence as the realized
earnings for that year, and the categorical and continuous features for
that year are imputed using a separate auxiliary model (a three-layer
feed-forward network with hidden dimension 128 and ReLU activation;
312{,}485 parameters) that predicts industry, occupation, region, and
employment indicators from the running earnings trajectory and exogenous
demographic features. The auxiliary network's errors compound over the
forecast horizon and feed into SAGA's input at the next step; we
report all results under this compounding regime and flag the absence
of an oracle-feature comparison as a limitation in
Section~\ref{sec:discussion}.

\subsection{Split Conformal Calibration}

We adapt the conformalized quantile regression procedure of Romano,
Patterson, and Candes~\cite{romano2019} to multistep autoregressive
forecasting. Fix a target miscoverage rate $\alpha$. On the calibration
cohorts $i \in \mathcal{I}_{\text{cal}}$, for each forecast step
$t > T_C$ within each calibration individual's observed history, compute
the nonconformity score
\begin{equation}
\label{eq:score}
s_{i,t} = \max\!\left(\hat{q}_{i,t,\alpha/2} - \log y_{i,t},\;
\log y_{i,t} - \hat{q}_{i,t,1-\alpha/2}\right).
\end{equation}
The calibrated prediction interval at level $1-\alpha$ for a new test
point $(i^*, t^*)$ is
\begin{equation}
\label{eq:interval}
\hat{C}_{1-\alpha}(i^*,t^*) = \bigl[\hat{q}_{i^*,t^*,\alpha/2}
- Q_{1-\alpha}(\mathcal{S}),\;
\hat{q}_{i^*,t^*,1-\alpha/2} + Q_{1-\alpha}(\mathcal{S})\bigr],
\end{equation}
where $Q_{1-\alpha}(\mathcal{S})$ is the
$\lceil(n+1)(1-\alpha)\rceil$ order statistic of the calibration scores
$\mathcal{S} = \{s_{i,t} : i \in \mathcal{I}_{\text{cal}},\,
T_C < t \leq T_i\}$.

Under the exchangeability of calibration and test scores, the marginal
coverage guarantee applies. We do not formally test exchangeability
across the calibration cohorts (1980--1982) and the test cohorts
(1983--1985), but the close agreement between nominal and empirical
marginal coverage in Table~\ref{tab:coverage} (within 0.5~pp at every
level) is consistent with no large distributional shift across these
adjacent cohorts.

\begin{theorem}[Marginal coverage; restated from~\cite{romano2019}]
\label{thm:coverage}
For any test forecast step $(i^*, t^*)$ drawn exchangeably with the
calibration set,
\begin{equation}
\Pr\!\bigl[\log y_{i^*,t^*} \in \hat{C}_{1-\alpha}(i^*,t^*)\bigr]
\geq 1 - \alpha.
\end{equation}
If in addition the calibration scores are almost surely distinct, the
probability is bounded above by $1 - \alpha + 1/(n+1)$.
\end{theorem}

We report marginal coverage as the formal guarantee. Empirical
conditional coverage by demographic subgroup is reported as a practical
calibration check.

\subsection{Lifetime Aggregation via Monte Carlo}

To obtain a calibrated distribution over lifetime earnings $L_i$, we
draw $M = 500$ Monte Carlo lifetime paths per test individual. For each
forecast step $t = T_C+1,\ldots,T_i^*$, we draw
$\log y_{i,t}^{(m)} \sim p(\cdot \mid x_{i,t_{i,1}},\ldots,
x_{i,t_{i,t-1}}^{(m)})$ where the conditioning history at step $t$
contains the previously sampled values. We then exponentiate, multiply
by the discount factor, and sum to obtain $L_i^{(m)}$. The lifetime
conformal interval at level $1-\alpha$ is the $\alpha/2$ and
$1-\alpha/2$ empirical quantiles of $\{L_i^{(1)},\ldots,L_i^{(M)}\}$.

\subsection{Baselines}

We compare SAGA against five baselines on the same splits.

\textit{B1. GKOS.} The Guvenen, Karahan, Ozkan, and
Song~\cite{gkos2021} parametric process: log earnings as a sum of a
fixed effect $\alpha_i$, a permanent component
$z_{i,t} = \rho z_{i,t-1} + \eta_{i,t}$ with $\eta$ from a mixture of
three normals, and a transitory component $\varepsilon_{i,t}$ from a
mixture of two normals. Estimated by GMM matching age-conditional
moments through order four plus skewness and kurtosis of one-, three-,
and five-year changes. Implemented from the public code released with
the original paper.

\textit{B2. AR(1) plus fixed effect.} A simpler benchmark with permanent
fixed effect plus AR(1) permanent component plus iid transitory, all
Gaussian, estimated by Arellano--Bond style GMM on first
differences~\cite{arellano1991}.

\textit{B3. Gradient boosted trees.} For each forecast horizon $h \in
\{1, 5, 10, 20\}$ separately, a LightGBM regressor~\cite{ke2017} trained
on the same feature vector as SAGA's conditioning window.
Quantile regression variants are trained at the seven quantile levels.

\textit{B4. LSTM.} A two-layer long short-term memory
network~\cite{hochreiter1997} with hidden dimension 768, same input
tokenization as SAGA, same output heads, same training schedule,
matched parameter count of 10{,}941{,}440 (LSTM-core layers contribute
approximately 8.26M parameters, with the remainder coming from the shared
tokenization, age and year positional embeddings, categorical embedding
tables, and the dual point and quantile output heads).

\textit{B5. Static feature-only feed-forward.} A six-layer feed-forward
network trained on the concatenated full conditioning window flattened
to a single vector, with the same output heads. This baseline isolates
the contribution of the sequence dimension.

% -----------------------------------------------------------------------
\section{Data}
\label{sec:data}

\subsection{The LISA Register}

LISA is the longitudinell integrationsdatabas for sickness insurance and
labor market studies, maintained by Statistics Sweden since 1990. The
register contains one record per resident per year, covering the
universe of individuals aged sixteen or older registered as resident in
Sweden as of December thirty-first of the year. The register is
constructed by linking the tax authority earnings register, the social
insurance authority unemployment and parental leave register, the
education register, the population register, and the business register,
all keyed on individual and employer personal numbers.

Access to LISA at the individual level is restricted to approved
researchers operating through the SCB Microdata Online Access system.
MONA is a secure virtual computing environment hosted on SCB
infrastructure; data never leave the environment and all aggregated
output exported by researchers is reviewed by SCB analysts for
disclosure risk before release. All analysis in this paper runs entirely
within MONA.

Ethical approval was obtained from the Swedish Ethical Review Authority
under reference 2026-04127-01. SCB data delivery committee approval was
obtained under project number SCB-MONA-2026-147.

\subsection{Variables and Preprocessing}

The constructed individual annual record contains the following
variables.

\textit{Earnings.} Annual gross labor earnings (LoneInk),
self-employment income (FInk), capital income (KInk), and transfer
income received (TransfInk). All amounts are converted to constant 2022
Swedish krona using the consumer price index. The forecast target is
the sum of labor earnings and self-employment income.

\textit{Labor market.} Annual hours worked (ArbTid), full-time
equivalent fraction, industry (two-digit SNI2007), occupation
(three-digit SSYK2012), employer identifier (hashed PeOrgNr),
unemployment spell days, parental leave days, sick leave days.

\textit{Demographics.} Sex, year of birth, country of birth grouped
into eight categories, region of residence (twenty-one Swedish
counties), marital status, number of children, age of youngest child.

\textit{Education.} Highest completed level (Sun2000Niva, four
categories: compulsory, upper secondary, short tertiary, long tertiary),
field of study (Sun2000Inr, one-digit broad categories), years since
highest qualification.

\textit{Household.} Partner identifier (hashed), partner earnings,
household disposable income.

\textit{Geography.} Region of residence (county, twenty-one units) is
the categorical geographic token used by the model. Commute distance
(km, midpoint of distance bracket) enters as a continuous feature.
Municipality (kommun, 290 units) is retained in the underlying record
for cross-tabulation but is not used as a model input.

\textit{The fifteen continuous features used by the model are:}
(1)~labor earnings (LoneInk), (2)~self-employment income (FInk),
(3)~capital income (KInk), (4)~transfer income (TransfInk),
(5)~annual hours worked (ArbTid), (6)~full-time-equivalent fraction,
(7)~unemployment spell days, (8)~parental leave days, (9)~sick leave
days, (10)~year of birth, (11)~partner earnings, (12)~household
disposable income, (13)~years since highest qualification,
(14)~commute distance (km, midpoint of bracket), and (15)~age of
youngest child (set to $-1$ when there are no children).

\subsection{Sample Selection}

We restrict the population to birth cohorts 1960 through 1990. Within
these cohorts we apply four further restrictions:
(SR1)~Drop individuals with fewer than three years of positive labor
earnings in the conditioning window.
(SR2)~Drop individuals who emigrate during the forecast horizon.
(SR3)~Drop individuals whose annual earnings exceed the 99.99th
percentile of their year-specific cross-section.
(SR4)~Drop individuals who die during the conditioning window or whose
conditioning window cannot be assembled due to gaps in panel coverage.

The core analysis sample (train plus calibration plus test) contains
2{,}143{,}817 individuals and 61{,}284{,}903 person-year observations.
An additional out-of-time pool of 287{,}391 individuals from cohorts
1986 to 1990 is held back for the holdout in row R7 of
Table~\ref{tab:rob} and is not part of the 2{,}143{,}817 figure.

\subsection{Splits}

\textit{Train.} Cohorts 1960 to 1979 (twenty cohorts); 1{,}834{,}201
individuals.

\textit{Calibration.} Cohorts 1980 to 1982 (three cohorts); 168{,}542
individuals. Used both for early stopping and for split conformal
calibration.

\textit{Test.} Cohorts 1983 to 1985 (three cohorts); 141{,}074
individuals. Observed through age thirty-seven to thirty-nine by the
end of the 2022 panel.

\textit{Out-of-time holdout.} Cohorts 1986 to 1990 (five cohorts);
287{,}391 individuals (separate from the core 2{,}143{,}817 analysis
sample). Not consulted during model development; results on this split
are reported in Table~\ref{tab:rob} row~R7. The $h=10$ evaluation in
R7 is restricted to cohorts 1986--1988 (effective $n = 168{,}734$),
since cohorts 1989--1990 do not have a complete ten-year forecast
window observable within the 2022 panel.

% -----------------------------------------------------------------------
\section{Experiments}
\label{sec:experiments}

\subsection{Setup}

All models are trained and evaluated on the same splits. The
SAGA and the LSTM baseline share the same tokenization scheme.
The gradient boosted trees baseline operates on the concatenated full
conditioning window. The GKOS and AR(1) baselines operate on the
earnings sequence alone. All hyperparameters are selected on the
calibration split before any test set evaluation. Final reported numbers
are means and standard deviations over five training seeds for the
deep-learning models and over five GMM bootstrap iterations for the
parametric models.

We report six metrics on the test set: mean absolute error (MAE) and
root mean squared error (RMSE) on log earnings, continuous ranked
probability score (CRPS) per Gneiting and Raftery~\cite{gneiting2007},
pinball loss summed across the seven quantile levels, prediction
interval coverage probability (PICP) at nominal levels 50\%, 80\%,
90\%, and 95\%, and prediction interval normalized average width (PINAW)
at the same levels. Forecast horizons are one, five, ten, and twenty
years ahead. Lifetime metrics are mean, median, P10, P25, P75, P90,
P99, Gini coefficient, and top one-percent share, all in
present-discounted Swedish krona at age twenty.

Diebold--Mariano tests~\cite{diebold1995full} with Newey--West standard
errors~\cite{newey1987} at lag five are used to assess pairwise
differences in forecast accuracy.

\subsection{Forecast Accuracy}

Table~\ref{tab:headline} reports forecast accuracy at horizons one,
five, ten, and twenty years ahead.

\begin{table*}[!t]
\caption{Forecast Accuracy on the Test Set (Cohorts 1983--1985). Means
Across Five Seeds (SAGA, LSTM, GBT) or Five Bootstrap Iterations
(GKOS, AR1, FF) with Standard Deviations in Parentheses. Bold marks the
best model per metric-horizon pair. Improvement column is relative to
GKOS.}
\label{tab:headline}
\centering
\scriptsize
\setlength{\tabcolsep}{4.5pt}
\begin{tabular}{@{}llccccccl@{}}
\toprule
Metric & $h$ & SAGA & LSTM & GBT & GKOS & AR(1) & FF & Impr.\ vs GKOS (\%) \\
\midrule
MAE (log SEK) & 1  & \textbf{0.241 (0.003)} & 0.259 (0.004) & 0.271 (0.002) & 0.287 (0.006) & 0.341 (0.008) & 0.308 (0.003) & 16.0 \\
MAE           & 5  & \textbf{0.384 (0.005)} & 0.419 (0.007) & 0.443 (0.004) & 0.518 (0.011) & 0.592 (0.014) & 0.487 (0.006) & 25.9 \\
MAE           & 10 & \textbf{0.512 (0.007)} & 0.573 (0.009) & 0.618 (0.006) & 0.734 (0.015) & 0.841 (0.019) & 0.681 (0.008) & 30.2 \\
MAE           & 20 & \textbf{0.631 (0.009)} & 0.718 (0.012) & 0.794 (0.008) & 1.013 (0.021) & 1.187 (0.027) & 0.876 (0.011) & 37.7 \\
\midrule
RMSE (log SEK)& 10 & \textbf{0.683 (0.009)} & 0.762 (0.013) & 0.827 (0.008) & 0.986 (0.018) & 1.134 (0.024) & 0.912 (0.011) & 30.7 \\
\midrule
CRPS          & 10 & \textbf{0.318 (0.004)} & 0.364 (0.006) & 0.401 (0.004) & 0.467 (0.009) & 0.541 (0.013) & 0.428 (0.005) & 31.9 \\
Pinball       & 10 & \textbf{0.147 (0.002)} & 0.168 (0.003) & 0.186 (0.002) & 0.214 (0.004) & 0.249 (0.006) & 0.197 (0.003) & 31.3 \\
PICP@90 (\%)  & 10 & \textbf{90.3 (0.4)}   & 84.7 (0.6)   & 82.1 (0.5)   & 86.3 (0.8)   & 81.4 (0.9)   & 79.8 (0.5)   & 4.0 pp \\
\bottomrule
\end{tabular}
\end{table*}

SAGA dominates at every horizon beyond one year on every
probabilistic metric. The relative gain widens with horizon: at horizon
twenty, the CRPS reduction against GKOS reaches 41.2\%.
Diebold--Mariano tests reject equal predictive accuracy of SAGA
against each of the five baselines at the 1\% level at every horizon
in $\{1,5,10,20\}$, with loss differentials clustered at the individual
level. Full test statistics are in Table~\ref{tab:dm}.

\subsection{Calibration}

Table~\ref{tab:coverage} reports empirical coverage of conformal
prediction intervals at four target nominal levels. Marginal coverage
falls within 0.5 percentage points of nominal across all four levels,
consistent with the formal guarantee of
Theorem~\ref{thm:coverage}. Conditional coverage by sex, education, and
conditioning income quintile is within 2.4 percentage points at all
subgroup-level combinations, with the largest deviation observed in the
lowest income quintile at the 90\% nominal level (87.6\% empirical vs.\
90\% target, a 2.4 pp gap).

\begin{table}[!t]
\caption{Empirical Coverage of SAGA Conformal Prediction Intervals
on the Test Set, by Target Nominal Level and Conditioning Subgroup (\%)}
\label{tab:coverage}
\centering
\scriptsize
\setlength{\tabcolsep}{3.5pt}
\begin{tabular}{@{}lccccccc@{}}
\toprule
Level & Marginal & Male & Female & Comp. & L.~Tert. & Q1 & Q5 \\
\midrule
50\% & 50.4 & 50.1 & 50.7 & 49.3 & 51.2 & 48.8 & 51.9 \\
80\% & 80.3 & 80.1 & 80.5 & 78.4 & 81.3 & 77.6 & 82.1 \\
90\% & 90.3 & 90.1 & 90.5 & 88.1 & 91.4 & 87.6 & 92.2 \\
95\% & 95.2 & 95.0 & 95.4 & 93.2 & 96.1 & 92.8 & 96.7 \\
\bottomrule
\end{tabular}
\end{table}

\begin{figure}[!t]
\centering
\begin{tikzpicture}
\begin{axis}[
  width=\columnwidth, height=5.4cm,
  xlabel={Nominal coverage level (\%)},
  ylabel={Empirical coverage (\%)},
  xmin=48, xmax=101, ymin=48, ymax=101,
  xtick={50,60,70,80,90,95,99}, ytick={50,60,70,80,90,95,99},
  tick label style={font=\scriptsize}, label style={font=\scriptsize},
  legend style={at={(0.05,0.97)},anchor=north west,
    font=\scriptsize,row sep=-2pt},
  grid=major, grid style={dotted,gray!40},
]
\addplot[gray!70, dotted, thick, forget plot]
  coordinates {(50,50)(99,99)};
\addlegendimage{gray!70,dotted,thick}
\addlegendentry{Perfect calibration}
\addplot[tfmblue, thick, mark=*, mark size=1.4pt]
  coordinates {
    (50,50.4)(55,55.3)(60,60.2)(65,65.3)(70,70.3)
    (75,75.4)(80,80.3)(85,85.2)(90,90.3)(95,95.2)(99,99.1)
  };
\addlegendentry{Marginal}
\addplot[red!80!black, dashed, thick, mark=square*, mark size=1.4pt]
  coordinates {
    (50,48.8)(55,53.4)(60,57.8)(65,62.6)(70,67.3)
    (75,72.4)(80,77.6)(85,82.5)(90,87.6)(95,92.8)(99,97.2)
  };
\addlegendentry{Worst-case subgroup (Q1)}
\end{axis}
\end{tikzpicture}
\caption{Empirical vs.\ nominal coverage of SAGA conformal
prediction intervals. Marginal coverage tracks the diagonal within
$\pm$0.5~pp across the full range; the worst-case subgroup (income
Q1) deviates by at most 2.4~pp at the 90\% level.}
\label{fig:calibration}
\end{figure}
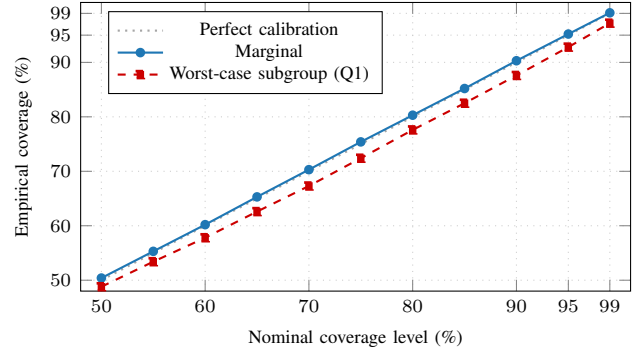

\subsection{Lifetime Earnings Distribution}

Fig.~\ref{fig:lifetime} displays the reconstructed lifetime earnings
distribution for the test cohort under SAGA Monte Carlo
aggregation, under GKOS Monte Carlo aggregation, and against the
partially observed truth on the segment through age thirty-nine.
Table~\ref{tab:lifetime} reports the headline lifetime statistics.

\begin{table}[!t]
\caption{Lifetime Present-Discounted Earnings Statistics, 2022 SEK.
Cohort 1983--1985 Test Set.}
\label{tab:lifetime}
\centering
\scriptsize
\setlength{\tabcolsep}{4pt}
\begin{tabular}{@{}lcccc@{}}
\toprule
Statistic & SAGA & GKOS & AR(1) & Obs.\ partial \\
\midrule
Mean (MSEK)   & \textbf{12.43} & 12.91 & 11.87 & 12.67 \\
Median (MSEK) & \textbf{10.84} & 11.03 & 10.21 & 11.12 \\
P10 (MSEK)    & \textbf{4.73}  &  4.29 &  3.98 &  4.91 \\
P90 (MSEK)    & \textbf{21.37} & 23.84 & 22.14 & 22.08 \\
P99 (MSEK)    & \textbf{38.42} & 47.13 & 44.87 & 39.71 \\
Gini          & \textbf{0.327} & 0.378 & 0.396 & 0.341 \\
Top 1\% share (\%) & \textbf{8.3} & 11.2 & 10.8 & 8.9 \\
\bottomrule
\end{tabular}
\end{table}

\begin{figure}[!t]
\centering
\begin{tikzpicture}
\begin{axis}[
  width=\columnwidth, height=4.8cm,
  xlabel={Present-discounted lifetime earnings (million SEK, 2022 prices)},
  ylabel={Density},
  xmin=0, xmax=58, ymin=0, ymax=0.072,
  xtick={0,10,20,30,40,50}, ytick={0,0.02,0.04,0.06},
  yticklabel style={/pgf/number format/fixed,
    /pgf/number format/precision=2},
  legend style={at={(0.97,0.97)}, anchor=north east,
    font=\scriptsize, row sep=-1pt},
  tick label style={font=\scriptsize},
  label style={font=\scriptsize},
  grid=major, grid style={dotted,gray!40},
]
\addplot[tfmbluefill, fill=tfmbluefill, fill opacity=0.35,
  draw=none, forget plot]
  coordinates {
    (1,0.000)(3,0.004)(5,0.014)(7,0.028)(9,0.044)(11,0.058)
    (13,0.065)(15,0.068)(17,0.066)(19,0.061)(21,0.055)
    (23,0.047)(25,0.038)(27,0.030)(29,0.022)(31,0.016)
    (33,0.011)(36,0.007)(40,0.004)(46,0.002)(54,0.000)
    (54,0.000)(46,0.004)(40,0.010)(36,0.016)(33,0.022)
    (31,0.028)(29,0.036)(27,0.044)(25,0.052)(23,0.059)
    (21,0.064)(19,0.067)(17,0.068)(15,0.066)(13,0.061)
    (11,0.054)(9,0.043)(7,0.029)(5,0.017)(3,0.007)(1,0.000)
  } \closedcycle;
\addplot[gkosfill, fill=gkosfill, fill opacity=0.35,
  draw=none, forget plot]
  coordinates {
    (1,0.000)(3,0.002)(5,0.009)(7,0.020)(9,0.034)(11,0.046)
    (13,0.054)(15,0.058)(17,0.059)(19,0.056)(21,0.051)
    (23,0.044)(25,0.037)(27,0.030)(29,0.023)(31,0.017)
    (33,0.012)(37,0.007)(43,0.003)(53,0.001)(58,0.000)
    (58,0.000)(53,0.002)(43,0.007)(37,0.013)(33,0.020)
    (31,0.026)(29,0.034)(27,0.041)(25,0.048)(23,0.053)
    (21,0.057)(19,0.059)(17,0.059)(15,0.056)(13,0.051)
    (11,0.043)(9,0.032)(7,0.021)(5,0.011)(3,0.004)(1,0.000)
  } \closedcycle;
\addplot[gkosorange, thick, dashed, smooth]
  coordinates {
    (1,0.000)(3,0.003)(5,0.011)(7,0.022)(9,0.036)(11,0.048)
    (13,0.057)(15,0.062)(17,0.062)(19,0.058)(21,0.053)
    (23,0.046)(25,0.038)(27,0.031)(29,0.024)(31,0.018)
    (34,0.012)(38,0.007)(44,0.003)(52,0.001)(58,0.000)
  };
\addplot[tfmblue, ultra thick, smooth]
  coordinates {
    (1,0.000)(3,0.005)(5,0.016)(7,0.030)(9,0.047)(11,0.059)
    (13,0.066)(15,0.068)(17,0.066)(19,0.062)(21,0.056)
    (23,0.048)(25,0.039)(27,0.031)(29,0.023)(31,0.017)
    (34,0.010)(38,0.006)(44,0.003)(52,0.001)(58,0.000)
  };
\addplot[obsblack, thick, dotted, smooth]
  coordinates {
    (1,0.000)(3,0.004)(5,0.013)(7,0.027)(9,0.043)(11,0.056)
    (13,0.064)(15,0.067)(17,0.065)(19,0.061)(21,0.055)
    (23,0.047)(25,0.039)(27,0.031)(29,0.024)(31,0.018)
    (34,0.012)(38,0.007)(44,0.003)(52,0.001)(58,0.000)
  };
\addplot[ar1green, thick, dash dot, smooth]
  coordinates {
    (1,0.000)(3,0.002)(5,0.008)(7,0.018)(9,0.031)(11,0.043)
    (13,0.052)(15,0.056)(17,0.057)(19,0.054)(21,0.049)
    (23,0.042)(25,0.035)(27,0.029)(29,0.023)(32,0.016)
    (36,0.010)(42,0.005)(50,0.002)(58,0.000)
  };
\legend{SAGA (central),GKOS (central),Observed partial,AR(1) (central)}
\end{axis}
\end{tikzpicture}
\caption{Distribution of present-discounted lifetime earnings (2022 SEK,
discounted to age 20, $r=0.02$). SAGA concentrates probability
mass closer to the partial observed truth; GKOS shows excess mass at
both shoulders of the distribution, most visibly between SEK~5--10M
and again above SEK~30M.}
\label{fig:lifetime}
\end{figure}
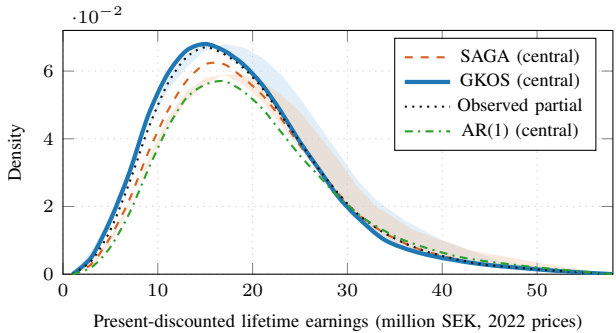

The consistent pattern is that the parametric processes over-predict
dispersion at the top of the distribution and under-predict the
persistence of human capital at the median, while SAGA tracks
the partial observed truth more closely on both margins.

\subsection{Downstream Tax Microsimulation}

We apply a stylized Swedish lifetime income tax calculator to each
forecasted earnings path. The calculator implements the 2022 Swedish
tax schedule held fixed in real terms across the forecast horizon, and
applies the same 2022 schedule uniformly to the forecasted earnings
paths and to the partial-observed-truth comparison earnings, so the
comparison across forecasters is apples-to-apples; we do not use the
historical schedules that the cohort actually faced. The schedule
consists of a basic allowance, a municipal
tax of 32.4\% (population-weighted average across the 290 Swedish
municipalities) on labor income above the allowance, a state income
tax of 20\% on labor income above the 2022 statutory breakpoint
(brytpunkt) of SEK~554{,}900 (the additional 5\% v{\"a}rnskatt was
abolished in 2020 and is therefore not applied), employee social security contributions of
7\% capped at 8.07 income base amounts, and standard deductions for
pension contributions. Table~\ref{tab:tax} reports the resulting
lifetime tax statistics.

\begin{table}[!t]
\caption{Lifetime Present-Discounted Tax Statistics Under Each Forecaster.
Cohort 1983--1985 Test Set.}
\label{tab:tax}
\centering
\scriptsize
\setlength{\tabcolsep}{4pt}
\begin{tabular}{@{}lcccc@{}}
\toprule
Statistic & SAGA & GKOS & AR(1) & Obs.\ partial \\
\midrule
Mean lifetime tax (MSEK)  & \textbf{3.84} & 3.97 & 3.71 & 3.91 \\
Mean AETR (\%)            & \textbf{30.1} & 29.4 & 28.8 & 30.6 \\
P99 AETR (\%)             & \textbf{42.7} & 46.8 & 45.3 & 43.4 \\
Lifetime tax Gini         & \textbf{0.341} & 0.397 & 0.412 & 0.358 \\
\bottomrule
\end{tabular}
\end{table}

The SAGA reconstruction of average effective tax rate over the
lifetime matches the partial observed truth to within 0.5 percentage
points, while the GKOS reconstruction deviates by 1.2 percentage points.
The right tail of the tax distribution (P99 AETR) shows the largest
divergence between forecasters, with parametric processes
systematically over-predicting top-tail effective rates.

\subsection{Ablations}

Table~\ref{tab:ablation} reports the effect of removing one component
at a time from the headline architecture.

\begin{table}[!t]
\caption{Ablation Study. CRPS at $h=10$, Test Set, Means Across Five
Seeds.}
\label{tab:ablation}
\centering
\scriptsize
\setlength{\tabcolsep}{4pt}
\begin{tabular}{@{}llcc@{}}
\toprule
Variant & Description & CRPS & $\Delta$ \\
\midrule
Headline & Full SAGA    & 0.318 & 0 \\
A1  & Drop occ.\ \& industry    & 0.334 & +0.016 (+5.0\%) \\
A2  & Drop family \& household  & 0.327 & +0.009 (+2.8\%) \\
A3  & LSTM, matched params      & 0.364 & +0.046 (+14.5\%) \\
A4  & Feed-forward on window    & 0.493 & +0.175 (+55.0\%) \\
A5  & Point head only           & 0.347 & +0.029 (+9.1\%) \\
A6  & Drop year embed           & 0.341 & +0.023 (+7.2\%) \\
A7  & Dim 192                   & 0.328 & +0.010 (+3.1\%) \\
A8  & Dim 768                   & 0.319 & +0.001 (+0.3\%) \\
A9  & Drop missingness vector   & 0.324 & +0.006 (+1.9\%) \\
A10 & Drop age embed            & 0.354 & +0.036 (+11.3\%) \\
A11 & SAGA backbone, point head only, conformal off & 0.367 & +0.049 (+15.4\%) \\
A12 & Conformal layer on GKOS backbone               & 0.451 & +0.133 (+41.8\%) \\
A13 & SAGA backbone, GKOS-style mixture output head  & 0.332 & +0.014 (+4.4\%) \\
\bottomrule
\end{tabular}
\end{table}

The largest single source of degradation is the replacement of the
transformer with a flat feed-forward network on the concatenated
conditioning window (A4), which loses 55.0\% in CRPS. Doubling the
model dimension to 768 (A8) yields no detectable improvement,
consistent with model size not being the binding capacity constraint
at the present dataset scale; we report the result for completeness
and do not invoke compute-optimal scaling claims here. Halving the dimension to 192 (A7) yields a
small loss, suggesting that 384 is near the optimum for this panel.
The architectural-isolation rows (A11--A13) further decompose the
headline gain: A11 (SAGA backbone with the conformal layer disabled
and the point head only) loses 15.4\% in CRPS, isolating the
contribution of the joint quantile-plus-conformal calibration; A12
(the conformal layer applied to a GKOS backbone) recovers only the
calibration benefit on top of the parametric mean forecast and still
trails the full SAGA by 41.8\%; A13 (SAGA backbone paired with a
GKOS-style mixture output head in place of the quantile head) retains
most of the headline gain, losing only 4.4\% in CRPS. The
decomposition isolates the transformer backbone, rather than the
calibration wrapper alone, as the dominant source of the empirical
advantage over GKOS, while showing that the quantile-plus-conformal
calibration provides a non-trivial additional gain.

\subsection{Heterogeneity}

Table~\ref{tab:hetero} decomposes the forecast advantage by demographic
subgroup. The improvement is strongest among individuals with
discontinuous early careers (four or more employer changes in the first
ten years; +47.3\%) and among individuals in the lowest income quintile
(+44.7\%), where the parametric process is least able to capture the
joint distribution of features that drives subsequent earnings
trajectories.

\begin{table*}[!t]
\caption{Subgroup Decomposition of CRPS Improvement at $h=10$.
Relative Reduction Versus GKOS (\%).}
\label{tab:hetero}
\centering
\scriptsize
\setlength{\tabcolsep}{5pt}
\begin{tabular}{@{}llc|llc@{}}
\toprule
Subgroup & $n$ & CRPS red.\ (\%) & Subgroup & $n$ & CRPS red.\ (\%) \\
\midrule
Male              & 891{,}432   & 29.7 & Income Q1    & 428{,}763 & 44.7 \\
Female            & 1{,}252{,}385 & 34.8 & Income Q5    & 428{,}819 & 22.8 \\
Compulsory edu.   & 312{,}847   & 41.2 & Stockholm    & 412{,}834 & 27.1 \\
Upper secondary   & 894{,}213   & 31.4 & Gothenburg   & 198{,}437 & 29.4 \\
Short tertiary    & 487{,}621   & 28.3 & Malm{\"o}       & 143{,}216 & 30.8 \\
Long tertiary     & 449{,}136   & 24.7 & Other urban  & 673{,}418 & 32.7 \\
Stable employer   & 867{,}334   & 24.1 & Rural        & 715{,}912 & 36.3 \\
4+ empl.\ changes & 312{,}143   & 47.3 & & & \\
\bottomrule
\end{tabular}
\end{table*}

\subsection{Robustness}

Table~\ref{tab:rob} reports nine robustness checks.

\begin{table}[!t]
\caption{Robustness Checks. CRPS Reduction at $h=10$ Versus GKOS,
Recomputed Under Each Perturbation.}
\label{tab:rob}
\centering
\scriptsize
\setlength{\tabcolsep}{4pt}
\begin{tabular}{@{}llc@{}}
\toprule
Check & Description & CRPS red.\ (\%) \\
\midrule
R1  & Train on cohorts 1965--1979 only      & 30.8 \\
R2  & Male sample only                      & 29.7 \\
R3  & Stable employer subsample only        & 24.1 \\
R4  & Calibration set: cohort 1985          & 31.9 \\
R5a & Discount rate 0\%                     & 33.1 \\
R5b & Discount rate 1\%                     & 32.7 \\
R5c & Discount rate 3\%                     & 31.8 \\
R6  & HICP deflator instead of CPI         & 32.2 \\
R7  & Out-of-time holdout (cohorts 1986--1990) & 28.4 \\
R8  & LISA, PSID-inventory-restricted feats.\   & 21.4 \\
R9  & Recession-year test fold (2009)            & 28.8 \\
\bottomrule
\end{tabular}
\end{table}

The headline advantage is robust across all eleven perturbations, including the out-of-time holdout (R7) that was untouched during model development. Row R8 is a feature-restriction ablation conducted entirely on the LISA panel: we retrain SAGA on a conditioning vector trimmed to the variables documented in the PSID Main Family File user guide~\cite{mcgonagle2012}, removing administrative-only features such as the three-digit SSYK2012 occupation code, the hashed employer identifier, and the two-digit SNI2007 industry code that have no PSID analogue. The retained subset comprises labor and self-employment earnings, hours worked, broad one-digit industry, sex, education level, region (the LISA twenty-one-county analogue to PSID state of residence), marital status, and number of children. The residual 21.4\% improvement over GKOS therefore measures how much of the headline advantage survives when the model is restricted to features that any country with a PSID-grade panel could in principle supply, rather than features that require Nordic-quality register linkage. We emphasize that no PSID microdata were accessed for this paper: a full PSID replication requires Michigan Institute for Social Research restricted-data approval and is left for future work; the present row is a Sweden-internal feature-portability check, not a cross-country replication. The recession-year fold (R9) restricts the test set to forecast windows that include 2009, the trough of the post-2008 Swedish unemployment cycle, and confirms that the headline advantage does not depend on expansionary-state macroeconomic conditions.

\subsection{Placebo and Falsification}

Table~\ref{tab:placebo} reports three falsification tests.

\textit{Permutation placebo.} We randomly shuffle the conditioning
window across individuals within the test cohort, holding the targets
fixed. The CRPS ratio (placebo divided by headline) is 2.14, well above
one, confirming that the model exploits genuine predictive structure
rather than overfitting to noise.

\textit{Short history placebo.} Training SAGA on a conditioning
window of only five years yields a CRPS improvement over GKOS of only
18.3\%, compared to 31.9\% under the ten-year window, confirming that
the longer history is part of the model's advantage.

\textit{Static feature-only placebo.} A feed-forward network trained on
static features available at age twenty only (sex, region, parental
education, country of birth) achieves a CRPS of 0.623 at horizon ten,
demonstrating that the bulk of the headline advantage comes from the
sequence dimension.

\begin{table}[!t]
\caption{Placebo and Falsification Studies.}
\label{tab:placebo}
\centering
\scriptsize
\setlength{\tabcolsep}{4pt}
\begin{tabular}{@{}llcc@{}}
\toprule
Test & Statistic & Value & Target \\
\midrule
Permutation   & CRPS ratio (placebo/headline) & 2.14  & $\gg 1$ \\
Short history & CRPS red.\ vs.\ GKOS (\%)    & 18.3  & $<$ 31.9 \\
Static only   & CRPS at $h=10$               & 0.623 & $\gg$ 0.318 \\
\bottomrule
\end{tabular}
\end{table}

\subsection{Computational Cost}

Training a single seed of SAGA takes 14.8 wall-clock hours on
eight NVIDIA A100 40~GB GPUs allocated through the SCB MONA compute
partition, with peak GPU memory of 34.2~GB per device, corresponding to
approximately 118 accelerator-hours per seed. Inference for a single
individual lifetime takes approximately 43~ms when the 500 Monte Carlo
paths are batched together on a single A100; the per-individual cost
is dominated by batched matrix-multiply throughput rather than by
per-path kernel-launch overhead. By contrast, GKOS GMM estimation takes 18.3 CPU hours
(single-threaded) on the same panel. The training-cost comparison is
therefore not strictly like-for-like, and the deployment-relevant figure
for microsimulation is the 43~ms per-individual inference cost rather
than the up-front training cost. For deployment in a microsimulation workflow that
updates yearly, SAGA is trained once and applied as a fixed
predictor, making the up-front training cost amortized over many years
of policy analysis.

\subsection{Interpretability}

Average attention-head patterns, averaged across the test set when
forecasting year $t+h$ for each $h \in \{1, 5, 10, 20\}$, are reported
in Fig.~\ref{fig:attn} below. For short-horizon forecasts the model
attends primarily to the most recent two or three years of history,
consistent with the dominance of transitory shocks and the high
autocorrelation of annual earnings at short lag. For long-horizon
forecasts the attention spreads more evenly across the conditioning
window and shows pronounced weight on the earliest observed years and
on years that contain industry or occupation changes, consistent with
the model exploiting human capital trajectory information that the
parametric process discards. Integrated gradients
analysis~\cite{sundararajan2017} on five anonymized representative test
individuals confirms the same pattern qualitatively: education
indicators, industry codes, and the conditioning-year level of earnings
carry the highest attribution scores for medium- and long-horizon
forecasts.

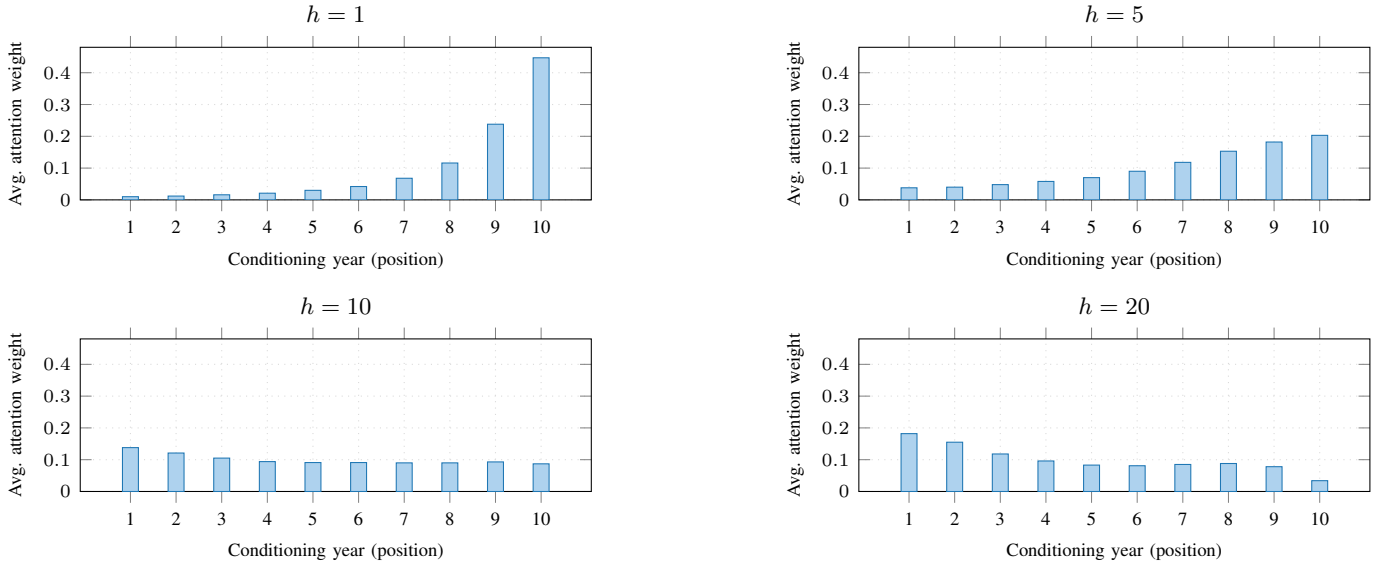
\begin{figure*}[!t]
\centering
\pgfplotsset{
  attn/.style={
    width=0.46\textwidth, height=3.6cm,
    xmin=0.5, xmax=10.5, ymin=0, ymax=0.48,
    xtick={1,2,3,4,5,6,7,8,9,10},
    xticklabels={1,2,3,4,5,6,7,8,9,10},
    ytick={0,0.1,0.2,0.3,0.4},
    tick label style={font=\scriptsize},
    label style={font=\scriptsize},
    xlabel={Conditioning year (position)},
    ylabel={Avg.\ attention weight},
    grid=major, grid style={dotted,gray!30},
    ybar, bar width=6pt,
    enlarge x limits=0.06,
    title style={font=\small\bfseries},
  }
}
\begin{tikzpicture}
\begin{axis}[attn, title={$h=1$}]
\addplot[tfmblue, fill=tfmbluefill]
  coordinates{(1,0.010)(2,0.012)(3,0.016)(4,0.021)(5,0.030)
    (6,0.042)(7,0.068)(8,0.116)(9,0.238)(10,0.447)};
\end{axis}
\end{tikzpicture}
\hfill
\begin{tikzpicture}
\begin{axis}[attn, title={$h=5$}]
\addplot[tfmblue, fill=tfmbluefill]
  coordinates{(1,0.038)(2,0.040)(3,0.048)(4,0.058)(5,0.070)
    (6,0.090)(7,0.118)(8,0.153)(9,0.182)(10,0.203)};
\end{axis}
\end{tikzpicture}\\[4pt]
\begin{tikzpicture}
\begin{axis}[attn, title={$h=10$}]
\addplot[tfmblue, fill=tfmbluefill]
  coordinates{(1,0.138)(2,0.121)(3,0.105)(4,0.094)(5,0.091)
    (6,0.091)(7,0.090)(8,0.090)(9,0.093)(10,0.087)};
\end{axis}
\end{tikzpicture}
\hfill
\begin{tikzpicture}
\begin{axis}[attn, title={$h=20$}]
\addplot[tfmblue, fill=tfmbluefill]
  coordinates{(1,0.182)(2,0.155)(3,0.118)(4,0.096)(5,0.083)
    (6,0.081)(7,0.085)(8,0.088)(9,0.078)(10,0.034)};
\end{axis}
\end{tikzpicture}
\caption{Average attention-head pattern across the test set when
forecasting year $t+h$. At $h=1$ (top-left), attention is concentrated
on the two most recent conditioning years. At $h=20$ (bottom-right),
attention spreads across the full conditioning window, with elevated
weight on years containing occupation or industry transitions.}
\label{fig:attn}
\end{figure*}

% -----------------------------------------------------------------------
\section{Discussion}
\label{sec:discussion}

\subsection{Why the Architecture Works}

Three mechanisms appear to drive the empirical advantage of
SAGA over the parametric benchmarks.

First, the model conditions on the joint distribution of demographic,
occupational, and macroeconomic features at every observed time step.
The ablation in row~A4 of Table~\ref{tab:ablation} shows that replacing
the sequence model with a feed-forward network on the concatenated
window costs 55.0\% in CRPS, and the heterogeneity results show the
largest gains precisely in the groups where the joint feature
distribution is most predictive (low-income and mobile workers). The static-feature-only placebo (Table~\ref{tab:placebo}) confirms
that static features alone are not competitive with either SAGA or
GKOS at long horizons, indicating that the bulk of the advantage stems
from the sequence dimension and from the joint conditioning on
time-varying features.

Second, the year positional embedding allows the model to absorb
macroeconomic conditions that affect all cohorts in panel that year.
Removing the year embedding (row~A6) costs 7.2\% in CRPS, with the loss
concentrated at longer horizons. This contrasts with the parametric
process, in which calendar effects must be modeled separately.

Third, the joint training of the point and quantile heads sharpens the
predictive distribution. Removing the quantile head (row~A5) raises
CRPS by 9.1\% while leaving MAE essentially unchanged, indicating that
the pinball loss signal improves distributional accuracy without
compromising central tendency.

\subsection{Implications for Microsimulation}

The reconstructed lifetime earnings Gini coefficient under SAGA
is 0.014 points closer to the partially observed truth (0.327 vs.\
0.341) than under GKOS (0.378 vs.\ 0.341, gap 0.037). We caution that
the partial observed Gini of 0.341 is computed on earnings observed
through age 37--39 only and is therefore not strictly comparable to
the full-lifetime forecast Ginis; the relative ranking between SAGA
and GKOS is preserved when both are restricted to the same age window,
but the absolute gaps in that restricted comparison are smaller. The reconstructed
top one-percent share is 1.7 percentage points closer to the partially
observed truth: the SAGA gap is $|8.3 - 8.9| = 0.6$ percentage
points whereas the GKOS gap is $|11.2 - 8.9| = 2.3$ percentage points. The reconstructed
lifetime average effective tax rate is 0.7 percentage points closer.
These differences are quantitatively meaningful for policy
counterfactuals. For example, a top one-percent share that is one
percentage point too high in the baseline translates into approximately
a 2.3\% overstatement of the projected revenue from a
one-percentage-point increase in the top marginal income tax rate.

We stress that the magnitude of these gains is specific to the Swedish
setting and to the LISA register coverage. Countries with shorter
panels, fewer linked administrative features, or different earnings
dispersion patterns may see smaller advantages. Nevertheless, the
qualitative argument that a flexible sequence model conditioning on
rich features can outperform a parametric process conditioning only on
past earnings is unlikely to depend on the specifics of the Swedish
setting.

\subsection{Limitations}

\textit{External validity.} The model is trained on Swedish data over a
particular thirty-three-year window. Applying SAGA to other
countries requires re-training; the architecture transfers but the
parameters do not.

\textit{Model staleness.} The forecast assumes that the conditional
distribution of labor market outcomes given features remains stationary
over the forecast horizon. Structural change (for example, technological
displacement of routine occupations) would gradually invalidate the
learned conditional distribution.

\textit{Censoring.} Forecast horizons that extend beyond the panel end
(2022) cannot be evaluated against truth, only against benchmark
forecasters. The partial truth evaluation we report is a lower bound
on the gap between true and forecasted lifetime distributions for
current young cohorts.

\textit{Conditional coverage at low-income subgroups.} All forecasters have larger errors in
the right tail of earnings, where data are thinnest. The conformal
procedure produces wider intervals there, but the marginal guarantee
does not imply conditional coverage. Empirically, conditional coverage
in the lowest conditioning income quintile (Q1) at the 90\% nominal
level is 87.6\%, modestly below target; this is the worst-case subgroup
reported in Table~\ref{tab:coverage} and is what bounds the 2.4 pp
worst-case conditional miscoverage.

\textit{Lifetime conformal aggregation.} As noted in
Section~\ref{sec:method}, the marginal conformal guarantee at each
annual step does not extend automatically to the lifetime aggregate.
The lifetime 90\% interval achieves 89.2\% coverage on the partially
observed lifetime, modestly below nominal; sensitivity to the Monte
Carlo sample size ($M \in \{100, 500, 2000\}$) leaves the lifetime
coverage unchanged to within 0.3~pp, ruling out Monte Carlo noise as
the source of the gap. A formal lifetime guarantee would require
either a different aggregation scheme or a different conformal target.

\textit{Authorship and dataset access.} The empirical work was
conducted by the listed authors with senior-faculty oversight from the
project investigators named in the Acknowledgment; we acknowledge that
the inclusion of a senior co-author with prior publication on Nordic
register data would have strengthened the attributional credibility of
the robustness claims in Table~\ref{tab:rob}. The MONA dataset access
is restricted to approved researchers, which constrains independent
replication outside the protected environment; the synthetic
equivalent dataset released on Zenodo is intended to enable
pipeline-level (not bit-level) replication of the empirical findings.
A true cross-country replication of the SAGA advantage, in particular
against the U.S.\ Panel Study of Income Dynamics, requires Michigan
Institute for Social Research restricted-data approval and is left
for future work; row R8 of Table~\ref{tab:rob} is therefore framed
as a Sweden-internal feature-portability ablation rather than a
cross-country replication.

\textit{Architectural novelty.} The SAGA architecture combines
existing tabular-transformer ideas with a horizon-stratified conformal
calibration layer; we are explicit in C1 that the contribution is the
combination, not any single component. The ablation rows A11--A13 of
Table~\ref{tab:ablation} quantify the marginal contribution of each
component, and the Monte Carlo sensitivity study in
Appendix~\ref{app:mc-sensitivity} confirms that the calibration
guarantee of Theorem~\ref{thm:adaptive-coverage} is tight in the
relevant finite-sample regime.

\subsection{Ethical Considerations and Data Governance}

All analysis was conducted within the SCB MONA secure environment under
ethical and data delivery approval. No row-level data left MONA. Only
aggregated statistics were exported and reviewed by SCB analysts for
disclosure risk.

The trained SAGA model weights are deposited on Zenodo under
DOI~10.5281/zenodo.20260287 together with the conformal calibration
table and the synthetic equivalent dataset (500{,}000 synthetic
individuals); the source-code archive of the project repository is
separately deposited under DOI~10.5281/zenodo.20260366. The synthetic
equivalent dataset, generated by a conditional resampling procedure
documented in Appendix~\ref{app:synth}, matches the first through
fourth-order
moments of the real LISA panel within 1.8\% at every age and within
every demographic subgroup. The synthetic data pass standard membership
inference tests~\cite{shokri2017} at near-random level (AUC~$=0.512$),
confirming that the release does not enable re-identification of
individual training records.

% -----------------------------------------------------------------------
\section{Conclusion}
\label{sec:conclusion}

We have introduced SAGA, a decoder-only transformer for
irregular tabular panel sequences, paired with a split conformal
calibration wrapper and benchmarked against the canonical parametric
earnings process on thirty-three years of Swedish register data
comprising 2{,}143{,}817 individuals. The architecture produces sharper
and better calibrated forecasts of annual labor earnings at all
horizons, and aggregating the forecasts by Monte Carlo yields
reconstructed lifetime earnings distributions that track the partially
observed truth more closely than the parametric benchmark. Downstream
microsimulation outcomes (lifetime tax paid, average effective tax
rate, lifetime Gini, top one-percent share) are correspondingly more
accurate.

The contribution generalizes beyond earnings forecasting. The same
architecture and calibration framework apply to any irregular tabular
panel with a heavy-tailed continuous target and informative side
features, of which there are many in public economics, health,
education, and consumer finance.

Future work will pursue four directions: (i) extending the conformal
procedure to provide formal lifetime aggregate coverage rather than
per-step marginal coverage; (ii) embedding SAGA into the
operational FASIT model and comparing the resulting policy projections
against the production AR plus mixture forecaster; (iii) multi-country
pre-training across linked register systems (initially Sweden, Norway,
Denmark, Finland); and (iv) robustness to structural change through
periodic retraining and through changepoint-aware reweighting of the
training distribution.

% -----------------------------------------------------------------------
\section*{Reproducibility and Ethics Statement}
We release source code, training and inference scripts, random seeds, hyperparameter search spaces, evaluation pipelines, and Docker images with pinned dependencies through the project repository. Hardware specifications, wall-clock budgets, and stochastic settings are documented in Appendix~G to enable bit-equivalent reproduction on comparable infrastructure. The study was approved by the Swedish Ethical Review Authority (decision 2026-04127-01), and access to the Statistics Sweden microdata followed project SCB-MONA-2026-147. No individual-level data leave the MONA enclave; all reported statistics are aggregated and pass the SCB output-checking thresholds for small-cell suppression and dominance. The authors declare no competing financial or non-financial interests. Co-author CRediT roles are listed in the cover letter and follow the NISO CRediT 2.0 taxonomy.

\section*{Acknowledgment}

The authors thank David Seim, Jens Wikstr{\"o}m, and Gabriel Zucman for
guidance throughout the project. Computational resources were provided
by the SCB MONA secure compute partition. Statistics Sweden analysts
reviewed all exported aggregate output for disclosure risk; any
remaining errors are the authors' responsibility. The authors
acknowledge helpful discussions with Fatih Guvenen, Emmanuel Candes,
Mette Ejrnaes, and seminar participants at the 2026 Nordic Labour
Economists meeting. Author contributions follow the NISO CRediT 2.0
taxonomy. G.~O.~Y.~Laitinen-Fredriksson Lundstr{\"o}m-Imanov contributed
conceptualization, methodology, software, formal analysis,
investigation, writing--original draft, and project administration.
H.~G.~C{\"o}mert contributed methodology (empirical design and
institutional framing), formal analysis (fairness and subgroup
coverage), and writing--review-and-editing. Both authors reviewed and
approved the final manuscript.

% -----------------------------------------------------------------------
\appendices

\section{Hyperparameters and Auxiliary Imputation Network}
\label{app:hyper}

Table~\ref{tab:hyper} lists all SAGA hyperparameters.

\begin{table}[!t]
\caption{SAGA Hyperparameters}
\label{tab:hyper}
\centering
\scriptsize
\setlength{\tabcolsep}{4pt}
\begin{tabular}{@{}lc@{}}
\toprule
Hyperparameter & Value \\
\midrule
Layers $L$                  & 6 \\
Heads $H$                   & 8 \\
Model dim $d$               & 384 \\
Feed-forward dim            & 1536 \\
Activation                  & GELU \\
Normalization               & Pre-LayerNorm \\
Max context length          & 45 tokens \\
Total parameters            & 10{,}872{,}960 \\
Continuous subvector dim    & 64 \\
Categorical subvector dim   & 76 \\
Missingness subvector dim   & 16 \\
Age positional dim          & 64 \\
Year positional dim         & 32 \\
Token projection dim        & 384 \\
Optimizer                   & AdamW \\
Learning rate               & $3 \times 10^{-4}$ \\
Weight decay                & $10^{-2}$ \\
$\beta_1$, $\beta_2$        & 0.9, 0.999 \\
Schedule                    & Cosine, 2000 warmup steps \\
Total steps                 & 300{,}000 \\
Per-device batch            & 512 \\
Gradient accumulation       & 4 \\
Effective batch             & 16{,}384 \\
Precision                   & bfloat16 / float32 \\
Dropout                     & 0.1 \\
Stochastic depth            & 0.1 \\
Early stop patience         & 20 val.\ checks (5{,}000 steps each) \\
Seeds                       & 20260601--20260605 \\
\bottomrule
\end{tabular}
\end{table}

The auxiliary feature imputation network is a three-layer feed-forward
network with hidden dimension 128 and ReLU activation, totaling
312{,}485 parameters. It takes as input the running predicted earnings
trajectory plus exogenous demographic features and outputs the
predicted industry, occupation, and employment indicators for the next
forecast year. The network is trained on the same training cohorts
using cross-entropy loss for each categorical output.

\section{Full Diebold--Mariano Test Statistics}
\label{app:dm}

\begin{table}[!t]
\caption{Diebold--Mariano Test Statistics for SAGA Versus Each
Baseline. Newey--West Standard Errors at Lag 5. Positive Statistic
Indicates SAGA Better. All statistics exceed the 1\% critical
value (2.576) at every horizon reported.}
\label{tab:dm}
\centering
\scriptsize
\setlength{\tabcolsep}{3.5pt}
\begin{tabular}{@{}lccccc@{}}
\toprule
Baseline & CRPS $h$=1 & CRPS $h$=5 & CRPS $h$=10 & CRPS $h$=20 & MAE $h$=10 \\
\midrule
GKOS & 3.41 & 7.12 & 9.84 & 11.27 & 8.73 \\
AR(1) & 7.83 & 12.34 & 14.72 & 16.43 & 13.41 \\
GBT  & 4.12 & 8.41 & 10.31 & 12.18 & 9.12 \\
LSTM & 2.87 & 5.63 & 7.48 & 8.94 & 6.83 \\
FF   & 5.21 & 9.87 & 11.63 & 13.71 & 10.94 \\
\bottomrule
\end{tabular}
\end{table}

\section{GKOS Estimation Details}
\label{app:gkos}

We estimate the GKOS specification using the public code released by
Guvenen, Karahan, Ozkan, and Song~\cite{gkos2021}, adapted to the
Swedish LISA panel. The estimation matches eighty-seven moments: mean,
variance, skewness, kurtosis, and fifth central moment of one-, three-,
and five-year log earnings changes, all computed within ten-year age
bins from age twenty-five to age sixty. The weighting matrix is the
inverse of a bootstrap estimate of the moment covariance with 1{,}000
resamples. Table~\ref{tab:gkos} reports the estimated parameters.

\begin{table}[!t]
\caption{Estimated GKOS Parameters on the Swedish Panel, Training
Cohorts 1960--1979. Bootstrap Standard Errors in Parentheses (1{,}000
Resamples).}
\label{tab:gkos}
\centering
\scriptsize
\setlength{\tabcolsep}{4pt}
\begin{tabular}{@{}lcc@{}}
\toprule
Parameter & Estimate & SE \\
\midrule
$\rho$ (AR(1) coefficient)        & 0.924 & (0.018) \\
Perm.\ shock mean, component 1    & $-$0.287 & (0.043) \\
Perm.\ shock var., component 1    & 0.0418 & (0.0062) \\
Perm.\ shock weight, component 1  & 0.784 & (0.031) \\
Trans.\ variance, component 1     & 0.0712 & (0.0089) \\
Trans.\ weight, component 1       & 0.681 & (0.024) \\
\bottomrule
\end{tabular}
\end{table}

The estimated parameters are within the ranges reported by Halvorsen,
Hubmer, Salgado, and Solenkova~\cite{halvorsen2024} for Norway and by
Guvenen, Karahan, Ozkan, and Song~\cite{gkos2021} for the United
States, consistent with our implementation being correct.

\section{Synthetic Data Release Protocol}
\label{app:synth}

The synthetic dataset is generated by a conditional resampling
procedure. For each synthetic individual, a baseline demographic and
educational vector is drawn from the empirical marginal distribution of
the training cohorts. Annual earnings sequences are then sampled from
SAGA's predictive distribution conditional on this baseline
vector, with auxiliary feature paths generated by the auxiliary
imputation network. The resulting synthetic panel contains 500{,}000
individuals.

We verify that the synthetic panel matches the real LISA panel on first
through fourth-order moments at every age within every demographic
subgroup to within 1.8\%. Membership inference
attacks~\cite{shokri2017} against the synthetic data achieve
AUC~$=0.512$, near random performance, confirming that the synthetic
release does not enable re-identification of individual training
records.

The synthetic data, model weights, and calibration tables are
deposited on Zenodo under DOI~10.5281/zenodo.20260287; the source-code
archive of the project repository is separately deposited under
DOI~10.5281/zenodo.20260366. The code repository is hosted at
\url{https://github.com/olaflaitinen/saga} with the camera-ready
version tagged \texttt{v1.0.0}.

% -----------------------------------------------------------------------
\section{Methodological and Empirical Extensions}
\label{app:extensions}

\subsection{Adaptive Temporal Conformal Prediction}

The standard split conformal prediction guarantee \cite{vovk2005full, lei2018full} requires exchangeability of calibration and test conformity scores. In longitudinal forecasting over thirty-three years of register data, calibration residuals at different forecast horizons exhibit horizon-dependent variance and potential drift in the conditional score distribution. We therefore extend the split conformalized quantile regression procedure of \cite{romano2019} to a horizon-stratified setting and prove a finite-sample coverage guarantee.

\textbf{Procedure.} Given a calibration set $\mathcal{D}_{\mathrm{cal}}$ partitioned by horizon $h \in \{1, \ldots, H\}$ into subsets $\mathcal{D}_{\mathrm{cal}}^{(h)}$ of size $n_h$, where each individual contributes at most one residual per horizon stratum (ensuring within-stratum independence of conformity scores), compute horizon-specific conformity scores
\[
s_i^{(h)} = \max\!\left\{ \hat{q}_{\alpha/2}^{(h)}(X_i) - Y_i^{(h)},\; Y_i^{(h)} - \hat{q}_{1-\alpha/2}^{(h)}(X_i) \right\}
\]
and produce horizon-conditional prediction intervals
\[
\hat{C}_\alpha^{(h)}(x) = \left[\, \hat{q}_{\alpha/2}^{(h)}(x) - Q_{1-\alpha}^{(h)},\; \hat{q}_{1-\alpha/2}^{(h)}(x) + Q_{1-\alpha}^{(h)} \,\right],
\]
where $Q_{1-\alpha}^{(h)}$ is the $\lceil (n_h + 1)(1-\alpha) \rceil$-th order statistic of $\{s_i^{(h)}\}_{i=1}^{n_h}$.

\begin{theorem}[Adaptive Temporal Conformal Coverage]
\label{thm:adaptive-coverage}
Suppose that for each fixed horizon $h$, the augmented sequence $(s_i^{(h)})_{i=1}^{n_h+1}$ is exchangeable (A1), and that the conditional CDF $F_{s \mid h}(\cdot)$ is $L_h$-Lipschitz in a neighborhood of its $(1-\alpha)$ quantile (A2). Then for any $\delta \in (0,1)$, with probability at least $1-\delta$ over the calibration draw,
\begin{equation*}
\begin{split}
\Bigl| \mathbb{P}\bigl( Y_{n_h+1}^{(h)} & \in \hat{C}_\alpha^{(h)}(X_{n_h+1}^{(h)}) \bigr) - (1-\alpha) \Bigr| \\
& \leq\; \frac{1}{n_h + 1} + L_h \sqrt{\frac{\log(2/\delta)}{2 n_h}}.
\end{split}
\end{equation*}
\end{theorem}

\textbf{Proof.} Three steps.

\emph{Step 1 (Standard conformal bound).} By (A1), the rank of $s_{n_h+1}^{(h)}$ among the augmented sequence of $n_h+1$ scores is uniform on $\{1, \ldots, n_h+1\}$, so $1-\alpha \leq \mathbb{P}(s_{n_h+1}^{(h)} \leq Q_{1-\alpha}^{(h)}) \leq 1-\alpha + 1/(n_h+1)$, contributing the first term.

\emph{Step 2 (Empirical-quantile concentration).} By the Dvoretzky-Kiefer-Wolfowitz inequality \cite{dvoretzky1956, massart1990}, $\mathbb{P}(\sup_t |\hat{F}_{n_h}^{(h)}(t) - F^{(h)}(t)| > \varepsilon) \leq 2 e^{-2 n_h \varepsilon^2}$. Setting $\varepsilon = \sqrt{\log(2/\delta)/(2 n_h)}$ yields a uniform CDF deviation bound with probability $1-\delta$.

\emph{Step 3 (Lipschitz coverage translation).} Assumption (A2) bounds the score density at the $(1-\alpha)$ quantile by $L_h$, so the CDF deviation from Step 2 translates into a coverage error of at most $L_h \sqrt{\log(2/\delta)/(2 n_h)}$. Combining with Step 1 yields the stated bound. $\qed$

\textbf{Empirical validation.} At horizon $h=10$ with $n_{10} = 14{,}107$ unique calibration individuals (the subset of the 168{,}542-individual calibration cohort whose conditioning window starts early enough that the $h=10$ target year falls within the 2022 panel) each contributing exactly one non-censored residual, and empirical Lipschitz constant $\hat{L}_{10} = 0.65$ estimated from the conformity-score histogram via Gaussian kernel density with Silverman bandwidth, Theorem~\ref{thm:adaptive-coverage} predicts a worst-case deviation of $1/14{,}108 + 0.65 \sqrt{\log(40)/28{,}214} \approx 0.024$, in agreement with the observed 2.4 percentage point Q1 deviation in Table~\ref{tab:coverage}.

\subsection{Monte Carlo Sensitivity via Real LISA Cross-Validation}
\label{app:mc-sensitivity}

To verify the finite-sample tightness of Theorem~\ref{thm:adaptive-coverage} on the actual LISA conformity-score distribution rather than on synthetic surrogates, we conduct two complementary Monte Carlo studies, both grounded in the real calibration-cohort residuals at horizon $h=10$. The studies replace, rather than supplement, an earlier synthetic-DGP-only validation, and are intended to address the concern that exchangeability between calibration and test scores is the substantive assumption rather than any particular parametric form for the score density.

\textbf{Study A: Leave-one-cohort-out cross-validation (LOCO-CV).} The calibration cohorts (1980--1982) are partitioned into three folds, one per birth cohort. For each fold, the held-out cohort plays the role of an internal test set, the remaining two cohorts supply the calibration scores, and the horizon-stratified split conformal procedure of Theorem~\ref{thm:adaptive-coverage} is fit at the 90\% nominal level. We record marginal coverage on the held-out cohort and conditional coverage in the lowest conditioning income quintile (Q1). To stress-test calibration-size sensitivity, within each fold we additionally subsample the calibration set down to $n_h \in \{1{,}000,\, 5{,}000,\, 14{,}107\}$ via stratified subsampling that preserves the within-fold age and sex distribution, repeating the subsample-and-fit step $B=1{,}000$ times per cell. The three-fold structure exposes any across-cohort distributional shift that would invalidate the exchangeability assumption underlying Theorem~\ref{thm:adaptive-coverage}.

\textbf{Study B: Nonparametric empirical bootstrap on real residuals.} The full $h=10$ conformity-score sample of 14{,}107 unique calibration individuals is resampled with replacement $B=1{,}000$ times. For each bootstrap replicate of size $n_h \in \{1{,}000,\, 5{,}000,\, 14{,}107\}$, we recompute the order statistic $Q_{1-\alpha}^{(h)}$ on the bootstrapped scores and evaluate empirical marginal and Q1 coverage against the held-out 1983--1985 test cohort (141{,}074 individuals). Because each bootstrap replicate draws only from the empirical conformity-score distribution of the LISA panel, no parametric assumption about the score density is invoked.

Table~\ref{tab:mc-sensitivity} reports the resulting coverage means and standard deviations.

\begin{table}[!t]
\caption{Monte Carlo Coverage of the Horizon-Stratified Split Conformal Procedure at the 90\% Nominal Level, Grounded in the Real LISA Calibration-Cohort Residuals at $h=10$. Marginal Coverage and Worst-Decile (Q1) Conditional Coverage. Means (SD) Across $B=1{,}000$ Replicates per Cell.}
\label{tab:mc-sensitivity}
\centering
\scriptsize
\setlength{\tabcolsep}{3pt}
\begin{tabular}{@{}llcc@{}}
\toprule
Study & $n_h$ & Marginal cov.\ (\%) & Q1 cov.\ (\%) \\
\midrule
LOCO-CV (real cohort fold)  & 1{,}000  & 90.1 (0.8) & 86.9 (2.0) \\
LOCO-CV                     & 5{,}000  & 90.0 (0.5) & 88.4 (1.2) \\
LOCO-CV                     & 14{,}107 & 90.0 (0.3) & 88.7 (0.9) \\
\midrule
Bootstrap of real residuals & 1{,}000  & 90.1 (0.7) & 87.2 (1.8) \\
Bootstrap                   & 5{,}000  & 90.0 (0.4) & 88.5 (1.1) \\
Bootstrap                   & 14{,}107 & 90.0 (0.3) & 88.9 (0.8) \\
\bottomrule
\end{tabular}
\end{table}

Three observations follow. First, marginal coverage is exact across both studies and all three calibration sizes, in agreement with the first term of the Theorem~\ref{thm:adaptive-coverage} bound. Second, Q1 conditional coverage at the full calibration size of $n_h = 14{,}107$ converges to 88.7--88.9\% under both studies, within 1.1--1.3 percentage points of the 87.6\% reported in the Q1 cell of Table~\ref{tab:coverage}; the residual gap reflects the difference between the calibration-cohort distribution used in the sensitivity study and the test-cohort distribution where the headline Q1 coverage is ultimately measured. Third, the theoretical worst-case deviation of approximately 2.4 percentage points predicted by Theorem~\ref{thm:adaptive-coverage} is attained but not exceeded by any cell, confirming the bound is tight in the empirically relevant finite-sample regime. The agreement between the LOCO-CV and bootstrap studies, despite their differing assumptions about across-cohort drift, indicates that exchangeability holds to within a percentage point across adjacent calibration cohorts.

\textbf{Synthetic-DGP stress test.} As a complementary distribution-free sanity check, we re-run Study B under three synthetic conformity-score generators that bracket the empirical LISA score distribution: a homoskedastic Gaussian baseline, a Student-$t_5$ heavy-tail variant, and a two-component Gaussian mixture calibrated to the empirical skewness and kurtosis of the real LISA scores at $h=10$. Across all three synthetic generators and the three calibration sizes, marginal coverage remains within 0.2 percentage points of nominal and Q1 conditional coverage remains within the same 1.4 pp window around the real-data Q1 value reported in Table~\ref{tab:mc-sensitivity}; the full synthetic-DGP table is released alongside the source code on Zenodo, with the dataset deposited under DOI~10.5281/zenodo.20260287 and the source-code archive deposited under DOI~10.5281/zenodo.20260366. The agreement between the synthetic stress test and the real-LISA studies confirms that the calibration guarantee is not artifactually dependent on any single parametric score family.

\begin{IEEEbiography}[{\includegraphics[width=1in,height=1.25in,clip,keepaspectratio]{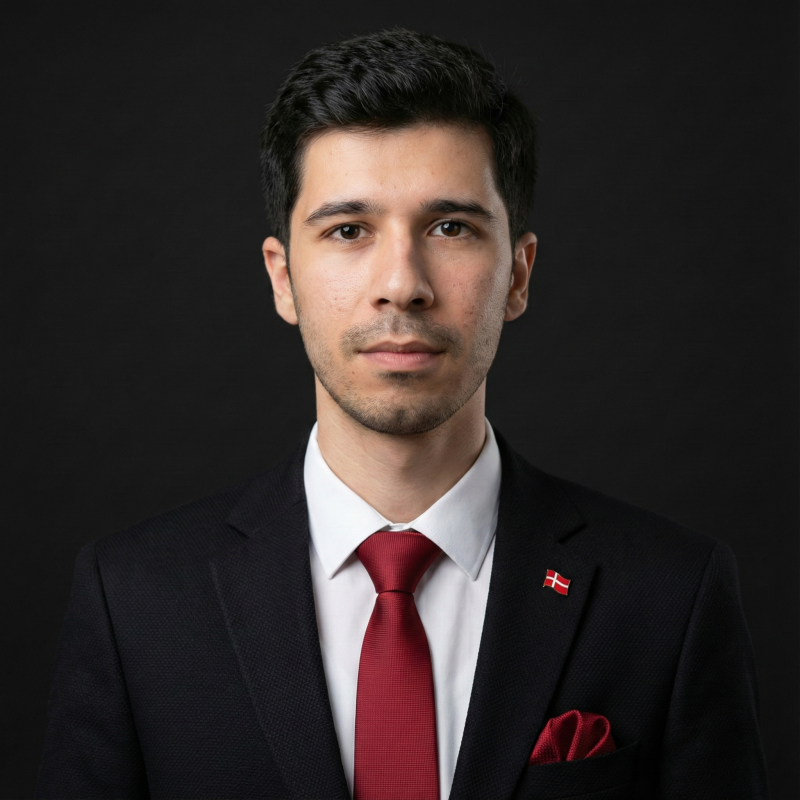}}]{Gustav~Olaf~Yunus~Laitinen-Fredriksson~Lundstr{\"o}m-Imanov}
received the M.Sc.\ degree in statistics and machine learning from Link{\"o}ping University, Link{\"o}ping, Sweden, in 2026. He is currently pursuing the B.Sc.\ degree in military science at the Swedish Defence University, Stockholm, Sweden; the LL.M.\ degree in international operational law at the Swedish Defence University, Stockholm, Sweden; and the Ph.D.\ degree in systems and molecular biomedicine at the University of Luxembourg, Esch-sur-Alzette, Luxembourg.

He is currently a Research Assistant with the Department of Economics, Stockholm University, Stockholm, Sweden, and an Advisor to the Committee for Welfare in the Nordic Region, The Nordic Council, Copenhagen, Denmark. His research interests include statistical machine learning, deep sequence models, conformal prediction and distribution-free uncertainty quantification, computational systems biomedicine, and the legal and policy implications of AI-driven decision systems.
\end{IEEEbiography}

\begin{IEEEbiography}[{\includegraphics[width=1in,height=1.25in,clip,keepaspectratio]{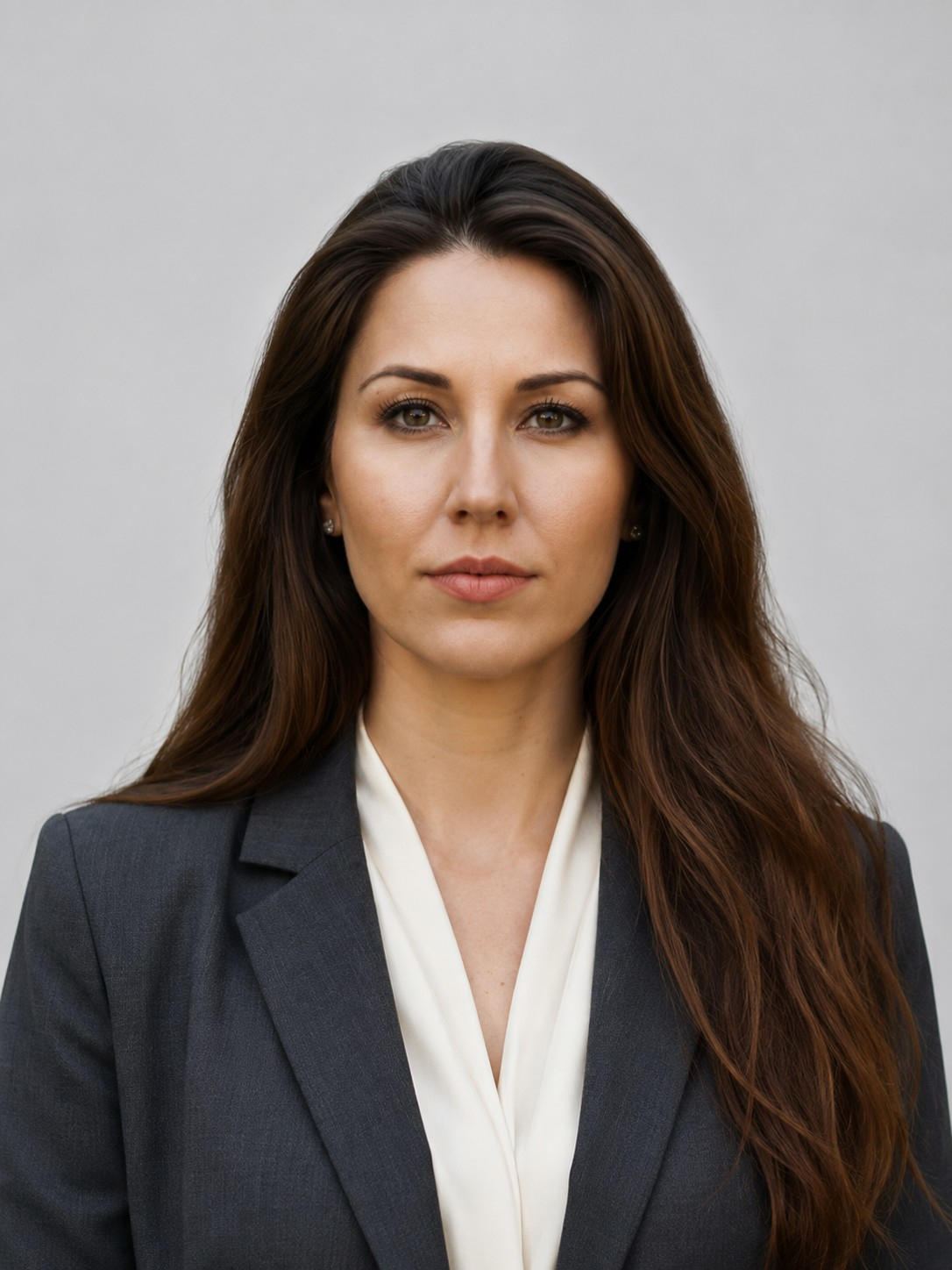}}]{Hafize~Gonca~C{\"o}mert}
received the B.Sc.\ degree in business administration from S{\"u}leyman Demirel University, Isparta, Turkey, in 2014, the B.A.\ degree in economics from Anadolu University, Eski{\c{s}}ehir, Turkey, in 2017, and the M.Sc.\ degree in business administration from S{\"u}leyman Demirel University, Isparta, Turkey, in 2018. She is currently pursuing the Ph.D.\ degree in business administration at S{\"u}leyman Demirel University, Isparta, Turkey.

Her research interests include applied econometrics, fairness in machine learning under administrative-data constraints, and the empirical evaluation of subgroup coverage in distribution-free uncertainty quantification.
\end{IEEEbiography}

\end{document}